\documentclass[letterpaper, 10 pt, conference]{ieeeconf}
\IEEEoverridecommandlockouts

\usepackage{cite}
\usepackage{amsmath,amssymb,amsfonts}
\usepackage{algorithm}
\usepackage{algpseudocode}
\usepackage{graphicx}
\usepackage{textcomp}
\usepackage{xcolor}
\usepackage{tabularx}

\usepackage{graphicx}
\usepackage{array}
\usepackage{caption}
\usepackage{subcaption}
\usepackage{adjustbox}
\usepackage[utf8]{inputenc}
\usepackage{mathtools}
\usepackage{balance}
\usepackage{hyperref}
\usepackage{authblk}

\def\Bc{{\mathcal B}}

\def\Dc{{\mathcal D}}

\def\Ec{{\mathcal E}}

\def\Gc{{\mathcal G}}

\def\Hc{{\mathcal H}}

\def\Kc{{\mathcal K}}

\def\Lc{{\mathcal L}}

\def\Mc{{\mathcal M}}

\def\Nc{{\mathcal N}}
\def\Nbb{{\mathbb N}}

\def\Oc{{\mathcal O}}

\def\Pc{{\mathcal P}}

\def\pbf{{\mathbf p}}

\def\Qc{{\mathcal Q}}

\def\Rbb{{\mathbb R}}

\def\Sc{{\mathcal S}}

\def\sbf{{\mathbf s}}

\def\Tc{{\mathcal T}}

\def\tbf{{\mathbf t}}

\def\Uc{{\mathcal U}}

\def\ubf{{\mathbf u}}

\def\Vc{{\mathcal V}}

\def\xbf{{\mathbf x}}

\def\ybf{{\mathbf y}}

\def\Zc{{\mathcal Z}}

\def\zbf{{\mathbf z}}

\def\ep{{\epsilon}}

\def\lam{{\lambda}}

\def\0{{\bf 0}}

\newcommand{\bitem}{\begin{itemize}}
\newcommand{\eitem}{\end{itemize}}
\newcommand{\btabular}{\begin{tabular}}
\newcommand{\etabular}{\end{tabular}}
\newcommand{\bcenter}{\begin{center}}
\newcommand{\ecenter}{\end{center}}
\newcommand{\bea}{\begin{eqnarray}}
\newcommand{\eea}{\end{eqnarray}}
\newcommand{\bean}{\begin{eqnarray*}}
\newcommand{\eean}{\end{eqnarray*}}

\newcommand{\ba}{\left. \begin{array}}
\newcommand{\ea}{\\ \end{array} \right.}
\newcommand{\bab}{\left[ \begin{array}}
\newcommand{\eab}{\\ \end{array} \right]}
\newcommand{\bap}{\left( \begin{array}}
\newcommand{\eap}{\\ \end{array} \right)}
\newcommand{\bbm}{ \begin{bmatrix}}
\newcommand{\ebm}{\\ \end{bmatrix} }
\newcommand{\bear}{\begin{array}}
\newcommand{\eear}{\\ \end{array}}

\newcommand{\wt}{\widetilde}

\newcommand{\bs}{\boldsymbol}

\newcommand{\non}{\nonumber}

\newcommand{\ra}{\rightarrow}

\font\myownfont=cmr17 scaled \magstep5
\def\psfancypar#1#2{\def\biginitial#1{{\myownfont#1}}%
  \def\makeinitial#1{\setbox8\hbox{\strut\vbox to 1.3ex
    {\hbox{\biginitial#1}\vskip -4pc plus 3.5pc minus 3.5pc}}}%
  \makeinitial#1%
  \ifdim\parindent>1.3\wd8\dimen8=\parindent
     \else\dimen8=1.3\wd8\fi
  \hangindent=\dimen8\hangafter=-2
  \noindent
  \strut\hskip-1\dimen8\box8{\sc#2}}%

\newcounter{subequation}
\def\beasub{\addtocounter{equation}{+1}
\setcounter{subequation}{\value{equation}}
\setcounter{equation}{0}
\renewcommand{\theequation}{\arabic{subequation}\alph{equation}}
\begin{eqnarray}}
\def\eeasub{\end{eqnarray}
\setcounter{equation}{\value{subequation}}
\renewcommand{\theequation}{\arabic{equation}}}

\DeclareMathOperator*{\argmax}{argmax}

\newtheorem{theorem}{Theorem}[section]
\newtheorem{lemma}[theorem]{Lemma}
\newtheorem{problem}[theorem]{Problem}

\newtheorem{remark}[theorem]{Remark}

\DeclareMathOperator*{\argmin}{arg\,min}
\DeclareMathOperator*{\logdet}{logdet}

\makeatletter
\newcommand{\linebreakand}{%
\end{@IEEEauthorhalign}
\hfill\mbox{}\par
\mbox{}\hfill\begin{@IEEEauthorhalign}
}
\makeatother

\title{\LARGE \bf Spatially temporally distributed informative path planning for multi-robot systems}

\author{Binh Nguyen, Linh Nguyen, Truong X. Nghiem, Hung La,\\
Jos\'e Baca, Pablo Rangel, Miguel Cid Montoya, Thang Nguyen
\thanks{$^\star$Corresponding author: Thang Nguyen (thang.nguyen@tamucc.edu)}
\thanks{This work was supported by the U.S. National Science Foundation (NSF) under grants NSF-CAREER: 1846513, NSF-PFI-TT: 1919127, and NSF-CAREER: 2238296, and by the Army Research Office under Grant Number W911NF-23-1-0186.}
\thanks{Binh Nguyen, Thang Nguyen, Jos\'e Baca, Miguel Cid Montoya, and Pablo Rangel are with the Department of Engineering, Texas A\&M University–Corpus Christi, Corpus Christi, TX 78412, USA 
}
\thanks{Truong X. Nghiem is with the School of Informatics, Computing, and Cyber Systems, Northern Arizona University, Flagstaff, AZ 86011, USA}
\thanks{Linh Nguyen is with the Institute of Innovation, Science and Sustainability, Federation University Australia, Churchill, VIC 3842, Australia}
\thanks{Hung M. La is with the Advanced Robotics and Automation (ARA) Lab, Department of
Computer Science and Engineering, University of Nevada, Reno, NV 89557, USA 
}%
}
\date{December 2023}

\begin{document}

\maketitle

\begin{abstract}
 This paper investigates the problem of 
 informative path planning for a mobile robotic sensor network in spatially temporally distributed mapping.
 The robots are able to gather noisy measurements from an area of interest during their movements to build a Gaussian Process (GP) model of 
 a spatio-temporal field.
  The model is then utilized to predict the spatio-temporal phenomenon at different points of interest.  
 To spatially and temporally navigate the group of robots so that they can optimally acquire maximal information gains while their connectivity is preserved, we propose a novel multi-step prediction informative path planning optimization strategy employing our newly defined local cost functions.
By 
using the dual decomposition method, it is 
feasible and practical to effectively solve the optimization problem in a distributed 
manner.
 The proposed method was validated through synthetic experiments 
 utilizing real-world data sets. 
 \end{abstract}

\section{Introduction}
Understanding natural phenomena is crucial in many fields of science and technology.
However, collecting data with stationary sensors is often costly and time-consuming.
Mobile robotic sensor networks (MRSN) offer a new way to generate such spatio-temporal (ST) data since MRSN can be quickly deployed and target data collection in areas of high information value. 
With the development of unmanned vehicles in all fields (ground, surface water, underwater, and air), this approach can solve many monitoring and observation tasks \cite{dunbabin2012robots}. Compared to sensing by stationary sensor nodes, the main challenge of using the data collected by MRSN to model ST phenomena is that the number of observation positions, hence the number of measurements taken at a sampling instant, is restricted by the limited number of robots and their mobility constraints.
Thus, successful ST mapping solutions with mobile robots must take into account these limitations. 

Recently, Gaussian Process Regression (GPR) has 
received significant interest as a technique for discovering ST data correlations.
GPR provides a fundamental framework for nonlinear non-parametric Bayesian inference widely used in soil organic matter mapping \cite{nguyen2021mobile}, temperature mapping~\cite{lin2020distributed} and leakage detection \cite{yan2023confident}.
The use of non-parametric models opens possibilities for mapping solutions to remain generic and flexible, since hyperparameters are able to be adjusted to create more accurate practical models for some specific applications. GPR also provide an estimation of forecast uncertainties and provides opportunities for future planning algorithms focusing on uncertainties. This spatial and temporal mapping technique GPR can be used in any situation in which a mobile robot is faced with a phenomenon that differs in time and space. For example, as an typical exploration task, GPR for spatial temporal maps can be established for understanding temperature, chemical concentration and water flow. In the field of robotics, it may be extremely valuable for precise control to have the ability to model and predict environmental disturbances and then to make appropriate strategies against the impacts of disturbances. 

With the help of MRSN, ST mapping has been intensively investigated to observe and model temporal changes of unknown environments.
The authors in \cite{singh2010modeling}
take advantage of mobile sensors to build a map of spatio-temporal phenomena via GPR; however, the restrictions on the movement of mobile sensors were not considered. 
The study in \cite{ma2017informative}
attempted to capture a slow-changing phenomenon in real-time operation, but it is assumed that the phenomenon is static during robot measurements.
Recently, the authors in \cite{sears2022mapping} presented spatial-temporal mapping with observations from a single robot traversing on a fixed-path design. 

Among robot planning methods in exploration tasks,
informative path planning (IPP) 
\cite{Schmid2020traj} has excellent performance, as future paths are generated by estimated environment models.
The core 
idea of the IPP is based on minimizing prediction uncertainties, 
which leads to designing optimal routes to collect measurements.
In the literature, the centralized IPP can be found in \cite{nguyen2015information,le2021efficient}, where the observed data is collected in a central unit to update a surrogate model. Then, optimal paths are computed and sent to each robot.
These works meet inherent restrictions since a tremendous amount of data collected by many robots possibly results in 
congestion in both communication and computation.
In recent years, several studies have been devoted to distributed IPP 
\cite{ding2024resource,binh2024ipp} with regard to GPR.
However, none of them takes into account the mapping of 
spatial-temporal phenomena and 
the dynamics of actual robots.
Furthermore, the cost functions of the IPP used in these studies are separated, i.e., each robot has its own cost function related only to its future path.
This setup ignores the cross-relation between the future paths of neighbor robots.

Motivated by the above discussion, this article presents \emph{a new distributed IPP approach for 
mapping spatial-temporal fields by using multiple robots}.
In other words, 
we propose a spatially and temporally distributed prediction scheme based on GPR while the connectivity of the robot team is preserved during 
their movements.
Another key contribution of this paper is \emph{the 
novel local cost functions for the IPP optimization problem with respect to the future paths of neighbor robots}.
The proposed spatially temporally distributed IPP approach 
was validated by mapping spatio-temporal temperature using a real-world dataset. 

The organization of this paper is as follows. Section II briefly presents models of mobile robots for monitoring a spatio-temporal fields.
The IPP optimality with connectivity preservation is then presented in Section III. 
Next, Section IV describes the distributed implementation of the proposed IPP algorithm.
Finally,
simulation results obtained by implementing the proposed approach 
using the real-world dataset 
in synthetic environments are 
discussed in Section
V. 
The conclusions 
are described in Section VI. 

{\bf Notations:}
Let us denote
$\Nbb$ and $\mathbb{R}$ as the sets of natural and  real numbers, respectively,  
$\otimes$ as the Kronecker product, and $\textbf{1}_n \in \Rbb^n$ as a vector in which each element is 1.
With a set of integers (index set) $\Zc = \{z_1, z_2,\dots, z_n \,|\, i_j \in \Nbb\}$, define
$\big[ \Mc_i \big]_{i \in \Zc }
=  \bbm \Mc_{z_1} \\ \vdots \\ \Mc_{z_n}  \ebm$ as a block of matrices with appropriate dimension or a vector of scalars $\Mc_i$.

\section{Mobile Robots for Monitoring Spatio-Temporal Fields}

Consider a convex set $\Qc \in \Rbb^\ell$ standing for an operation space of all robots.
Let us define $M$ as a number of robots working in $\Qc$, and we assume that the communication area of each robot $i$ at anytime 
is a ball (or a circle in 2D) centered at $\pbf_{i,k}$ with 
radius $R$. 
Here, $\pbf_{i,k}$ is the location of robot $i$ at time $t_k$. 
In this paper, the spatio-temporal field of interest is considered as a latent relationship
$z : (\Qc, \Rbb^+) \ra \Rbb$ mapping a location of measurement  in
$\Qc$ and its current time $t_k$ to a spatio-temporal phenomenon. 
The robot $i$ observes a noisy
measurement $y_{i,k} \in \Rbb$ of the spatio-temporal field $z$ at its current position for every time step.
In addition, robot movements are described as 
\begin{align}
	\pbf_{i,k+1} = A_k \pbf_{i,k} + B_k \ubf_{i,k},
	\label{sysdyn}
\end{align}
where $\ubf_{i,k}$ stands for the bounded control input ($\Vert \ubf_{i,k}\Vert_\infty \leq \delta_i$) of robot $i$ between two consecutive time steps  with $\delta_i$ being the maximum magnitude of control input $\ubf_{i,k}$.
Additionally, $A_k$ and $B_k$ represent the matrices obtained by linearizing the dynamics of the robot at time step $k$.

Based on the above setups, let us define that robots $i$ and $j$ are connected at step $k$ if $\Vert\pbf_{i,k}-\pbf_{j,k}\Vert_2\leq R$. Accordingly,
let $\Ec_k$ be the set of pairs of connected robots $(i, j)$ at time step $k$, that is,
$\Ec_k = \left\{(i,j)\in \Vc\times\Vc : \Vert \pbf_{i,k} - \pbf_{j,k}\Vert_2\leq R \right\}.    
$
Denote $\Vc = \{1,2,\dots, M\}$ as a set of indexed vertices in which each robot represents a vertex.
Then, let $\Gc_k$ be an undirected graph established by set of vertices $\Vc$ and edges $\Ec_k$. Note that  $\Gc_k$ varies over time.
The undirected graph $\Gc_k$ is connected if there exists at least a path between any pair of robots.
Robot $i$ is considered a neighbor of robot $j$  if they are connected.

Denote $ \Nc_{i,k}$ as a set of neighbors of robot $i$ and let
$y_{i,k}$ be a measured value of the robot $i$ at the time $t_k$ at location $\pbf_{i,k}$.
Denote $\Dc_{i,k}$ as a dataset of robot $i$ collected up to time $t_k$.
The local data set $\Dc_{i,k}$ can be decomposed from the sets $\Dc_{i,k}^y, \Dc_{i,k}^\pbf, \Dc_{i,k}^\tau$ of all measurements, locations, and timestamps. 
Consequently, the data exchange in robot $i$ is described by $\Dc_{i,k} = \left(\cup_{j \in \Nc_{i,k}^+ } \Dc_{j,k-1} \right) \cup \{y_{i,k},\pbf_{i,k}, t_k\}$ where $\Nc_{i,k}^+ = \Nc_{i,k} \cup \{i\}$. 
In this setup, each robot has a measurement model as follow
\begin{equation}
        y_{i,k} = z(\pbf_{i,k}, t_k) + \delta_{i,k},
    \end{equation}
    where $\delta_{i,k} \sim \Nc(0,\sigma_{i,k}^2) $
is an independent and identically distributed
zero-mean Gaussian noise with standard deviation $\sigma_{i,k} > 0$,
and $z \sim \Gc\Pc(\mu,\Kc(\pbf,\pbf^\prime,t,t^\prime))$ is the
random/latent variable with covariance funcion $\Kc$ and mean $\mu$ which can be set as a deterministic function (constant, polynomial or periodic) or determined by observed data such as neural network \cite{rizzo1994characterization}.

\section{Informative Path Planning with network connectivity preservation}

The movements of robots possibly disrupt the connectivity of sensor network. Thus, this section proposes a distributed algorithm to ensure that the connectivity of robot network is preserved in the next step.
To be specific, if the network is currently connected, then by maintaining some edges, the network will be connected in the next step. 
We then formulate the IPP optimization problem given the robot dynamics and connectivity constraints.

At the beginning, we recall previous results in \cite{binh2024ipp}.
Let us define
$
\wt\Ec_{i,k+1} = \big\{(v,n) \in \Nc_{i,k}^+ \times \Nc_{i,k}^+, v\ne n \big|
\Vert \pbf_{v,k+1} - \pbf_{n,k+1}\Vert_2\leq R \big\}  
$
 as a set of connection at $k+1$ established by neighbors of robot $i$ at time step $k$.
Accordingly, let $\wt\Gc_{i,k+1}$ be the sub-graph defined at time step $k+1$ induced by $(\Nc_{i,k}^+, \wt\Ec_{i,k+1})$.
In this paper, we assume that the graph $\Gc_0$ 
is connected at the initial time $t_0$.

\begin{lemma}[\cite{binh2024ipp}]\label{lem1}
	Suppose that $\Gc_k$ is connected. If $\wt\Gc_{i,k+1}$ is connected for all $i\in \Vc$, then $\Gc_{k+1}$ is also connected.
\end{lemma}

\begin{theorem} \label{pres}
Suppose that:
(i) $\Gc_k$ is connected;
(ii) at time step $k+1$, robot $i$ is connected with robots in $\Sc_{i,k}$ determined by \textit{Algorithm~\ref{Alg1}} for all $i \in \Vc$.
Then $\Gc_{k+1}$ is connected.
\end{theorem}
\begin{proof}
Based on Lemma \ref{lem1}, the proof of Theorem \ref{pres} follows the same steps with similar arguments as in the proof of \cite[Lemma 3.6]{binh2024ipp}.  
\end{proof}

\begin{algorithm}[t]
	\caption{Distributed Connectivity Preservation \label{Alg1}}
	\begin{algorithmic}[1]
		\Statex{{\bf Input} Set of neighbors $\Nc_{i,k}$, their positions $\pbf_{j,k}$, communication radius $R_i$.}
		\Statex{{\bf Output} Set of robots $\Sc_{i,k}$($\subset \Nc_{i,k}$) to be preserved.}
		\State{{\bf Initiate:}} $\Sc_{i,k} = \emptyset$
		\For{$j \in \Nc_{i,k}$}
        \State{check = {\bf true}}
		\For{$\ell \in \Nc_{i,k}\setminus\{j\}$, $c_k (\ell,j) > 0$}
	\If{$\Vert \pbf_{i,k} - \pbf_{\ell,k} \Vert_2, \Vert \pbf_{j,k} - \pbf_{\ell,k} \Vert_2 < \Vert \pbf_{i,k} - \pbf_{j,k}  \Vert_2 $}
        \State{check = {\bf false}} 
	\EndIf
		\EndFor
        \If{check = {\bf true}} $\Sc_{i,k} :=  \Sc_{i,k} \cup \{j\}$
        \EndIf
		\EndFor
	\end{algorithmic}
\end{algorithm}





Let $\ybf_{\Dc_{i,k}}$ be a vector of all measurements that the robot~$i$ took up to time $t_k$, then the vector of local
measurements $\ybf_{\Dc_{i,k}}$ follows a multivariate Gaussian
distribution in the following form
\begin{equation}
    \ybf_{\Dc_{i,k}} \sim \Gc\Pc(\bs\mu_{\Dc_{i,k}}, \Sigma_{\Dc_{i,k}})
\end{equation}
where $\bs\mu_{\Dc_{i,k}}$ is a mean vector with regard to the local dataset $\Dc_{i,k}$, $\Sigma_{\Dc_{i,k}} = \Kc(\Dc_{i,k}^\pbf,\Dc_{i,k}^\pbf,\Dc_{i,k}^\tau,\Dc_{i,k}^\tau) + \sigma_{i,k}^2 I$ denotes a covariance matrix with noise term.
%
In what follows, for unobserved locations of interest  
\begin{align}
\hat\pbf_{i,\Hc} = [\hat\pbf_{i,k+1}^\top,\dots, \hat\pbf_{i,k+H}^\top]^\top  \in \Rbb^{\ell\times H}  
\end{align}
corresponds to specific time instants $\tbf_\Hc = [t_{k+1}, \dots, t_{k+H}]^\top$ and an unobserved vector of latent variables  $\hat\zbf_{i,\Hc} = [z(\hat\pbf_{i,k+1},t_{k+1}),\dots, \hat z(\hat\pbf_{i,k+H},t_{k+H})]^\top$.
Here, $\Hc = \{0,1,\dots, H-1\}$ with $0 < H \in\Nbb$ represents the predictive horizon. Additionally, let us use $(\hat\bullet)$ for an unobserved vector of latent variables, locations, and their corresponding covariance matrices.
Then, following a multivariate Gaussian distribution, it has
\begin{equation}
    \bbm \ybf_{\Dc_{i,k}} \\  \hat{\zbf}_{i,\Hc} \ebm = \Gc\Pc\left( \bbm \bs\mu_{\Dc_{i,k}} \\ \hat{\bs\mu}_{i,\Hc} \ebm,\bbm \Sigma_{\Dc_{i,k}}&\hat\Sigma_{\Dc_{i,k}\Hc} \\ \hat\Sigma_{\Dc_{i,k}\Hc}^\top & \hat\Sigma_{i,\Hc} \ebm \right)
\end{equation}
in which matrices $\hat\Sigma_{\Dc_{i,k}\Hc}, \hat\Sigma_{i,\Hc}$ are obtained from a spatio-temporal covariance function $\Kc(\pbf,\pbf^\prime,t,t^\prime)$ with respect to the data set $\Dc_{i,k}$ and unobserved locations $\hat\pbf_{i,\Hc}$.
In addition, $\hat{\bs\mu}_{i,\Hc}$ denotes mean vectors with respect to  unobserved locations $\hat\pbf_{i,\Hc}$.
According to \cite{williams2006gaussian}, the conditional distribution of unobserved positions is:
\begin{equation}
\hat{\zbf}_{i,\Hc}|\Dc_{i,k} \sim \Gc\Pc(\hat \mu_{i,\Hc|\Dc_{i,k}}, \hat \Sigma_{i,\Hc|\Dc_{i,k}}),
\end{equation}
where matrices $\hat \mu_{\Hc|\Dc_{i,k}}, \hat \Sigma_{i,\Hc|\Dc_{i,k}}$ are given by the following regression:
\begin{align*}
&\hat \mu_{i,\Hc|\Dc_{i,k}} = \hat\mu_{i,\Hc} + \hat\Sigma_{\Dc_{i,k}\Hc}^\top  \Sigma_{\Dc_{i,k}}^{-1} (\ybf_{\Dc_{i,k}} - \mu_{\Dc_{i,k}}),
\\
&\hat\Sigma_{i,\Hc|\Dc_{i,k}} = \hat \Sigma_{i,\Hc} - \hat\Sigma_{\Dc_{i,k}\Hc}^\top  \Sigma_{\Dc_{i,k}}^{-1} \hat\Sigma_{\Dc_{i,k}\Hc}.
\end{align*}






Let $\Pc_{i,k} = \{\sbf_1, \sbf_2, \dots, \sbf_n\} \in \Qc^n$  be a
set of locations in which robot $i$ can move in $\Qc$
between $H$ consecutive samplings $t_k$ and $t_{k+H}$. Let $\Tc_{k+H} = \{ \tau_1, \tau_2, \dots, \tau_m \} \in [t_k,t_{k+H}]^m$ ($ t_k < \tau_1 < \tau_2 < \dots< \tau_m = t_{k+H}$ and $\{t_{k+1}, \dots,$ $t_{k+H}\} \in \Tc_{k+H}$) be a set of time stamps between $t_k$ and $t_{k+H}$. Denote $\Uc_{i,k+H} = \{\hat y(\sbf, \tau)| \sbf \in \Pc_{i,k}, \tau \in \Tc_{k+H} \}$ 
as a vector of latent variables corresponding to $\Pc_{i,k}$ and $\Tc_{k+H}$. 

\begin{problem}
Find the optimal path $\hat\pbf_{i,\Hc}$ of each robot $i$ in the mobile robot network at time step $t_k$, leading to the lowest
uncertainties at all unmeasured locations of interest
\begin{equation}
    \hat\pbf_{i,\Hc} = \argmin S(\Uc_{i,k+H}|\Dc_{i,k}^y, \hat\ybf_{i,\Hc}),
    \label{problem}
\end{equation}
where $S(\bullet)$ is the conditional entropy.
\end{problem}

By using the chain rule for conditional entropy \cite{cover1999elements}, we have
    $S(\Uc_{i,k+H}|\Dc_{i,k}^y, \hat\ybf_{i,\Hc}) =  
    S(\Uc_{i,k+H}, \hat\ybf_{i,\Hc}|\Dc_{i,k}^y) - S(\hat\ybf_{i,\Hc}|\Dc_{i,k}^y)$.
It should be noted that $\hat\ybf_{i,\Hc}$ is a vector of latent variables at locations $\hat\pbf_{i,\Hc}$ and time $t_{k+1},\dots, t_{k+H}$.
We assume that $\hat\pbf_{i,\Hc} \in \Pc_{i,k}$ then $\hat\ybf_{i,\Hc}$ is contained in $\Uc_{i,k+H}$ ($t_{k+1},\dots, t_{k+H}\in \Tc_{k+H}$).
Then, $S(\Uc_{i,k+H}, \hat\ybf_{i,\Hc}|\Dc_{i,k}^y) = S(\Uc_{i,k+H}|\Dc_{i,k}^y)$ is constant. 
Therefore, it can be
clearly seen that (\ref{problem}) is converted to $\hat\pbf_{i,\Hc} = \argmax S(\hat\ybf_{i,\Hc}|\Dc_{i,k}^y)$. The conditional entropy of a multivariate Gaussian
distribution of random variables $\hat\ybf_{i,\Hc}$ at unobserved locations $\hat\pbf_{i,\Hc}$ at time $\tbf_\Hc$ is given by a closed form \cite{cover1999elements}: $S(\hat\ybf_{i,\Hc}|\Dc_{i,k}^y) = \frac{1}{2} \logdet \hat \Sigma_{i,\Hc|\Dc_{i,k}} + \text{const}$.

\begin{figure}
    \centering
    \includegraphics[width = 1.05\linewidth]{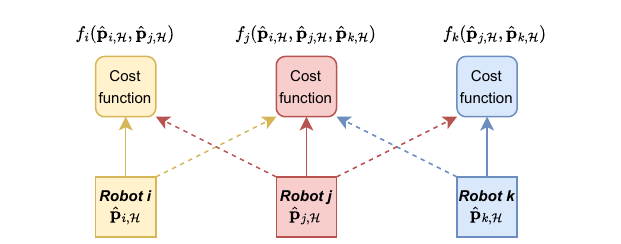}
    \caption{The architecture of local cost functions}
    \label{fig1}
    \vspace{-1em}
\end{figure}

If each robot determines its future path independently, the correlation between the optimal paths will be ignored.
The future paths are possibly determined to be close to each other.
Thus, we use virtual paths (a copy of the neighbor path) to construct the cost function as described in Fig. \ref{fig1}.
Let us define $\hat\pbf_{i,\Hc}^+ = \big[ \hat\pbf_{j,\Hc} \big]_{j\in \Nc_{i,k}^+}$.
This paper presents a cross-correlated local cost function of robot $i$ as $f_i(\hat\pbf_{i,\Hc}^+) = \logdet   \hat\Sigma_{i,\Hc|\Dc_{i,k}}(\hat\pbf_{i,\Hc}^+)$ which is associated with its future path and those of its neighbors.
As a result, the informative path planning for a multi-robot system in optimally mapping a spatio-temporal field is formulated in the following optimization problem: for all $h \in \Hc$
\begin{subequations}
\begin{align}
\max_{\hat\pbf_{i,\Hc}^+} &\sum_{i=1}^M \logdet   \hat\Sigma_{i,\Hc|\Dc_{i,k}}(\hat{\pbf}_{i,\Hc}^+),
\label{OP1}
\\
\text{s.t.~} 
 &\hat\pbf_{i,k+h+1} =  A_k\hat\pbf_{i,k+h} + B_k \ubf_{i,k+h},
 \\
 &\hat\pbf_{i,k} \in \Qc,~\Vert \ubf_{i,k+h} \Vert_\infty \leq  \delta_i,
\label{OP1_opti2}
\\
&\Vert \hat\pbf_{i,k+h} - \hat\pbf_{j,k+h} \Vert_2 \leq R,\, \forall j \in \Sc_{i,k}. \label{OP1_opti3}
\end{align}
\end{subequations}
\begin{remark}
Compared to the previous work \cite{binh2024ipp}, the correlation between the data collected by the robot $i$ and its neighbor's future paths is exploited. Therefore, each robot has an awareness of the future movements of its neighbors in its cost function.
\end{remark}


\section{Distributed IPP}
It should be noted that the objective function (\ref{OP1}) and the connectivity constraint (\ref{OP1_opti3}) are involved in at least two robots. To solve the optimization in a distributed way, each robot should create copies of its neighbors.
Let 
$\bs\zeta_{ij} = \big[ \bs\zeta_{ij,k+1}^\top,\dots, \bs\zeta_{ij,k+H} \big]^\top$ (for all $j \in \Nc_{i,k}$) represent the virtual positions 
of the robot $j$ 
estimated by the robot $i$. Denote a vector $\bs\zeta_i = \left[ \bs\zeta_{ij} \right]_{j\in \Nc_{i,k}^+}$ where $\bs\zeta_{ii} = \hat\pbf_{i,\Hc}$.
We tend to achieve $\bs\zeta_{ij,\Hc} = \hat\pbf_{j,\Hc}$ (for all $j \in \Nc_{i,k}$).
The optimization problem (\ref{OP1}) is equivalent to:
\begin{subequations}
\begin{align}
\max_{\bs\zeta_i} &\sum_{i=1}^M \logdet  \hat\Sigma_{i,\Hc|\Dc_{i,k}}(\bs\zeta_i),
\label{OP2}
\\
\text{s.t.~} &\bs\zeta_{ii} = \bs\zeta_{ji},  ~\forall (i,j)\in \Vc\times \Nc_{i,k},
\\
& \bs\zeta_{ii,k+h+ 1} = A_k \bs\zeta_{ii,k+h} + B_k \ubf_{i,k+h},
\\
&\bs\zeta_{ii} \in \Qc^H,~
\Vert \ubf_{i,k+h} \Vert_\infty \leq  \delta_i,
\label{org_opti2}
\\
&\Vert \bs\zeta_{ii,k+h} - \bs\zeta_{ij,k+h} \Vert_2 \leq R,  ~\forall (i,j)\in \Vc\times \Sc_{i,k}, \label{org_opti3}
\end{align}
\end{subequations}
for all $h\in \Hc$. 
For the sake of simplicity, we define $f_i(\bs\zeta_i) = -\logdet  \hat\Sigma_{i,\Hc|\Dc_{i,k}}(\bs\zeta_i)$.
The optimization problem (\ref{OP2}) can be solved in a distributed fashion by using proximal alternating direction method of multiplier (proximal ADMM) presented in \cite{le2021efficient,nguyen2023real}. However, the method requires gradient updates $\nabla f_i$ for every iteration. It should be noted that the computational complexity of $\nabla f_i$ is $\Oc(d|\Dc_{i,k}|^2)$ where $d$ is the input dimension and $|\Dc_{i,k}|$ is the number of local data, therefore updating $\nabla f_i$ at every iteration accounts for many computing resources.

The cost function (\ref{OP2}) is highly nonconvex, resulting in a great computational burden. Thus, let us approximate (convexify) the cost function around the previous value $\bs\zeta_i^{k-1}$ at step $k-1$
$\tilde f_i(\bs\zeta_i;\bs\zeta_i^{k-1}) = f_i(\xbf) \!+\! \nabla f_i^\top(\bs\zeta_i^{k-1})(\bs\zeta_i - \bs\zeta_i^{k-1}) \!+\! \frac{q_i}{2} 
\Vert \bs\zeta_i - \bs\zeta_i^{k-1} \Vert_2^2$
%
where $q_i > 0$ and $\nabla f_i(\bs\zeta_i^{k-1})$ represent the gradient of $f_i$ at $\bs\zeta_i^{k-1}$.
In the initial step, $\bs\zeta_i^{k-1}$ is selected from the initial position of the robots, that is, $\bs\zeta_{ii,h}^0 = \pbf_{i,0}$ and $\bs\zeta_{ij,h}^0 = \pbf_{j,0}$ for all $ h \in \Hc, j \in \Nc_{i,0}$. To simplify local constraints, let us define 
\begin{align*}
\Bc_{i,k} = \big\{
\bs\zbf = &\big[ \bs\zbf_{j} \big]_{j\in \Nc_{i,k}^+} \big| 
\bs\zbf_{j} = \left[ \bs\zbf_{j,h} \in \Qc \right]_{h = 1,\dots,H},
\\
&\zbf_{i,h+1} = A_k\zbf_{i,h} + B_k \ubf_{i,h} , \Vert  \ubf_{i,h}\Vert_\infty\leq \delta_i,
\\
\Vert &\bs\zbf_{i,h} \!-\! \bs\zbf_{j,h} \Vert_2 \leq R,~ j \in \Sc_{i,k}\big\}
\end{align*}
%
as a set of local constraints including robot dynamics and network connectivity.
At the initial step, positions $[{\bf 1}_H \otimes \pbf_{j,0}]_{j\in \Nc_{i,1}^+} \in \Bc_{i,0}$ for all $i$ because we already assumed that $\Gc_0$ is connected. Consequently, $\Bc_{i,0}$ is a non-empty set, and (\ref{OP1}) is feasible at the initial time step, and then $[{\bf 1}_H \otimes \pbf_{j,1}]_{j\in \Nc_{i,1}^+} \in \Bc_{i,1}$.
Sequentially with the next steps $k+1$, $\Bc_{i,k+1}$ is also nonempty. 
It can be observed that $\Bc_{i,k}$ is a convex set.
The optimization (\ref{OP2}) 
can be rewritten as
\begin{align}
\min \sum_{i=1}^N \tilde f_i(\bs\zeta_i), \text{~s.t.~} \bs\zeta_i \in  \Bc_{i,k},~ \bs\zeta_{ii} = \bs\zeta_{ji},
\label{OP3}
\end{align}
for all $i \in \Vc$ and $j\in \Nc_{i,k}$.
Next, the Lagrangian function of (\ref{OP3}) is defined by
$\Lc = \sum_{i=1}^{M} \Lc_i$,
where
\begin{align*}
\Lc_i = \tilde f_i(\bs\zeta_i) \!+\!
\sum_{j\in \Nc_{i,k}}^{} \bs\lambda_{ij}^\top( \bs\zeta_{ii} \!-\! \bs\zeta_{ji})
\end{align*}
where $\bs\lambda_i = [\bs\lambda_{ij}]_{j\in \Nc_{i,k}}$ is the dual variable. Then, using the dual decomposition method \cite{rush2012tutorial}, the optimization (\ref{OP3}) is handled by
\begin{align}
&\bs\zeta_1^{(n)}, \dots, \bs\zeta_N^{(n)} = \argmin_{\bs\zeta_i \in \Bc_{i,k}} \sum_{i=1}^{M} \Lc_i,
\label{AL1}
\\
&\bs\lambda_{ij}^{(n+1)} = \bs\lambda_{ij}^{(n)} - \alpha_n(\bs\zeta_{ii}^{(n)} - \bs\zeta_{ji}^{(n)}).
\label{AL2}
\end{align}
The optimization problem (\ref{AL1}) can be distributively solved. Indeed, the Lagrangian is written as $\Lc = \sum_{i=1}^N \big( \tilde  f_i(\bs\zeta_i) +\sum_{j\in \Nc_{i,k}} (  \bs\zeta_{ii}\bs\lambda_{ij} - \bs\lambda_{ji}\bs\zeta_{ij})\big) $. Accordingly,  the optimization (\ref{AL1}) is equivalent to
\begin{align}
    \bs\zeta_i^{(n)} &= \argmin_{\bs\zeta_i \in \Bc_{i,k}} \tilde  f_i(\bs\zeta_i) +\sum_{j\in \Nc_{i,k}} (  \bs\zeta_{ii}^\top\bs\lambda_{ij}^{(n)} -  \bs\zeta_{ij}^\top \bs\lambda_{ji}^{(n)}) \label{zeta_cal}
    \\
    &= \argmin_{\bs\zeta_i \in \Bc_{i,k}} (\nabla f_i(\bs\zeta_i^{k-1}) \!+\! \Lambda_i^{(n)} \!-\! q_i \bs\zeta_i^{k-1})^\top \bs\zeta_i \!+\! \frac{q_i}{2} 
    \non
\Vert \bs\zeta_i \Vert_2^2,
\end{align}
where $\Lambda_i^{(n)} = \left[ \sum_{j\in \Nc_{i,k}} \bs\lam_{ij}^{(n)\top},[-\bs\lam_{ji}^{(n)}]_{j\in \Nc_{i,k}}^\top\right]^\top$.
Note that (\ref{zeta_cal}) is formulated as a convex quadratic programming quadratic constraints (QCQP) that is solved effectively in polynomial time by solvers such as OSQP \cite{stellato2020osqp} or SOCP \cite{domahidi2013ecos}.
The distributed algorithm \ref{al2} is tailored to describe the steps to solve the optimization problem (\ref{OP3}).
\begin{algorithm}[t]
	\caption{Distributed solving of optimization \label{al2}) \label{alg:cap}}
	\begin{algorithmic}[1]
		\Statex{{\bf Input:} Number of robots $M$, set of neighbors $\Nc_{i,k}$, and a small tolerate error $\epsilon$.}
		\Statex{{\bf Output:} $\bs\zeta_{ii} = \left[ \hat \pbf_{i,t+h} \right]_{h \in \Hc}$}
		\State{{\bf Initiate:}
			$\bs\zeta_{ii}^{(0)} =  {\bf 1}_H \otimes \pbf_{i,0}$,  $\bs\zeta_{ij}^{(0)} =  {\bf 1}_H \otimes \pbf_{j,0}$
		}
		\Loop
		\State{Robot $i$ sends $\bs\lam_{ij}^{(n)}$ to $j$ and receives $\bs\lam_{ji}^{(n)}$ from $j$}
		\State{Compute $\bs\zeta_i^{(n+1)}$ by (\ref{zeta_cal})}
        \State{Robot $i$ sends $\bs\zeta_{ij}^{(n)}$ to $j$ and receives $\bs\zeta_{ji}^{(n)}$ from $j$}
        \State{Compute $\bs\lam_i^{(n+1)}$ by (\ref{AL2})}
		\If{$\max_{i\in \Vc,j \in \Nc_{i,k}} \Vert \bs\zeta_{ii}^{(n+1)} - \bs\zeta_{ij}^{(n+1)} \Vert_2 < \ep$} 
		\State{
			\Return $\bs\zeta_{ii}^{(n+1)}$
		}
		\EndIf
		\EndLoop
	\end{algorithmic}
\end{algorithm}

\section{Simulations and Discussions}
\begin{figure}[b]
    \centering
    \includegraphics[width=0.7\linewidth]{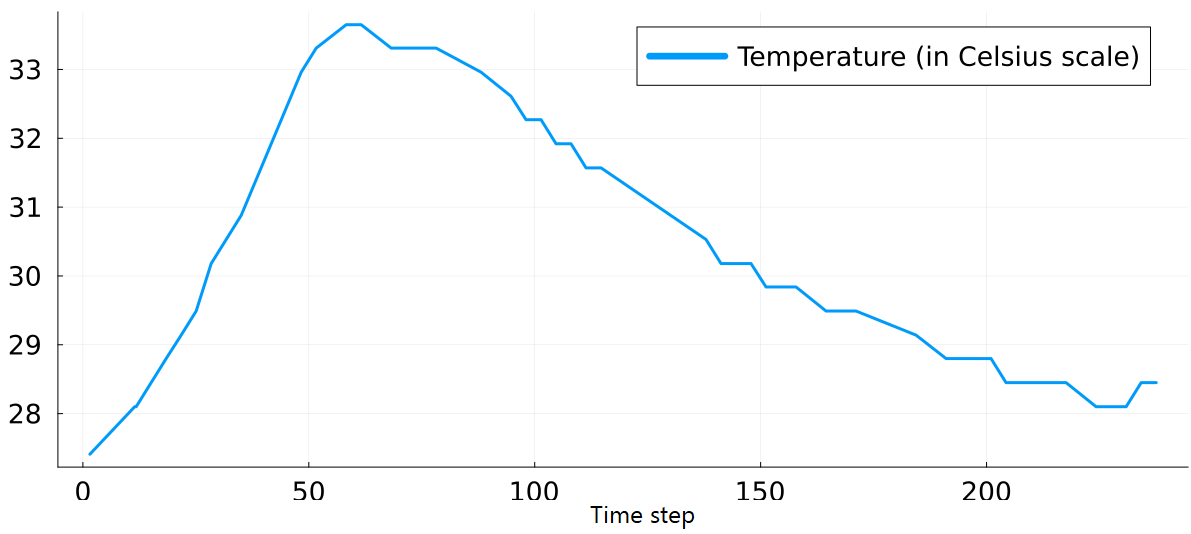}
    \caption{Temperature evolution measured by the $10^{th}$ sensor.}
    \label{sen10}
\end{figure}

To demonstrate the effectiveness of the proposed approach, we implemented it in a synthetic environment by using the real-world temperature dataset \cite{dataset_soil}. It is noted that the temperature dataset was spatially and temporally collected by 12 fixed-location sensors during 24 hours in a crop area of 20 m $\times$ 100 m, which resulted in 756 measurements in total. To exemplify the variation of the temperature overtime, we depict the data measured by the $10^{th}$ sensor in Fig. \ref{sen10}.
To verify our IPP algorithm, we first simulated a spatio-temporal field of the real-world temperature dataset by building a model from all 756 measurements. This model is called ground truth (GT). Then whenever a robot moves to a particular location in an unknown area and takes a virtual measurement at a particular time, the GT model would estimate that virtual measurement for the robot. In other words, the robots in our simulation virtually took the ``real'' measurements when they were exploring the field.

In the simulations, 6 robots with a communication range of $R = 20 \, [m]$ were chosen to conduct a task of mapping the spatio-temporal temperature field in an unknown area with the same dimensions of 20 m $\times$ 100 m. At the beginning, none of the robots knew anything about the field. They could only gather temperature information over their navigation.
In addition, let us take $\Vert \Delta v_{i,k}\Vert \leq 1\, [\rm{m/s}]$ and $\Vert \Delta \theta_{i,k}\Vert \leq 1 \, [\rm{rad}]$.
The covariance function 
was selected as a serial combination between square exponential and Matérn ($\frac{1}{2}$) functions 
as follows
\begin{align*}
\Kc(\pbf, \pbf^\prime, t, t^\prime) =   \sigma^2 \exp\left(\frac{\Vert \pbf- \pbf^\prime\Vert_2^2}{2\ell_s^2} + \frac{| t- t^\prime|}{\ell_t}\right),
\end{align*}
where $\ell_s$ and $\ell_t$ are spatial and temporal length scales, respectively.
We also considered wheeled mobile robots with the following dynamics 
\begin{align}
    \dot\pbf_i = v_i \bbm \cos \theta_i & \sin \theta_i\ebm^T
    \label{cycle}
\end{align}
where $v_i$ is the longitude velocity and $\theta_i$ is the heading angle of robot $i$. When discretizing (\ref{cycle}) by the Euler method, it has
$\pbf_{i,k+1} \approx \pbf_{i,k} + \tau v_{i,k} [ \cos \theta_{i,k},  \sin \theta_{i,k}]^\top$, where $\tau = 1 \, [s]$ is the sampling interval.
The dynamics of the robots are linearized by using first-order approximation as follows
\begin{align}
\hat\pbf_{i,k+h} = 2&\hat\pbf_{i,k+h-1} - \hat\pbf_{i,k+h-2}    
+\tau \bbm \cos\theta_{i,k-1} \\  \sin \theta_{i,k-1}\ebm \Delta v_{i,k} 
\non\\
& 
+ \tau\bbm -v_{i,k-1} \sin\theta_{i,k-1}
\\
 v_{i,k-1} \cos\theta_{i,k-1}
\ebm \Delta\theta_{i,k},
\end{align}
where $\Delta v_{i,k} = v_{i,k} - v_{i,k-1}$ and $\Delta \theta_{i,k} = \theta_{i,k} - \theta_{i,k-1}$. With respect to \eqref{sysdyn}, we have $A_k = I$ and $B_k = \tau \bbm \cos\theta_{i,k-1}  &-v_{i,k-1} \sin\theta_{i,k-1}\\ \sin \theta_{i,k-1} & v_{i,k-1} \cos\theta_{i,k-1}\ebm$, $\ubf_{i,k} = \bbm \Delta v_{i,k} \\  \Delta\theta_{i,k}\ebm$.
\begin{figure}[tb]
    \centering
    \begin{subfigure}[b]{0.485\linewidth}
        \includegraphics[width = \textwidth]{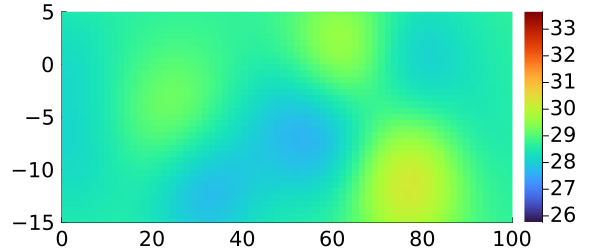}
        \caption{GT: Initial}
    \end{subfigure}
    \begin{subfigure}[b]{0.485\linewidth}
        \includegraphics[width = \textwidth]{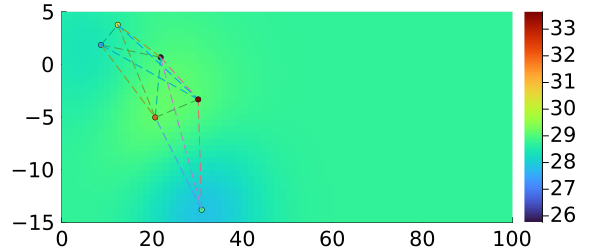}
        \caption{P: Initial}
    \end{subfigure}
    \begin{subfigure}[b]{0.485\linewidth}
        \includegraphics[width = \textwidth]{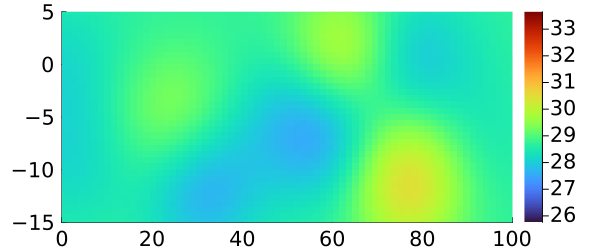}
        \caption{GT: at 20 steps}
    \end{subfigure}
    \begin{subfigure}[b]{0.485\linewidth}
        \includegraphics[width = \textwidth]{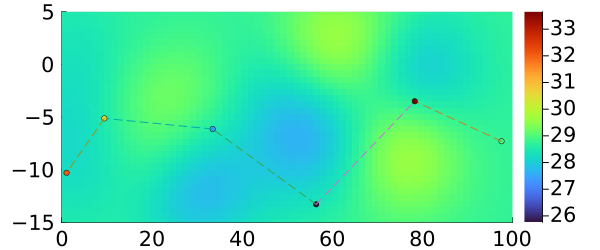}
        \caption{P: at 20 steps}
    \end{subfigure}
    \begin{subfigure}[b]{0.485\linewidth}
        \includegraphics[width = \textwidth]{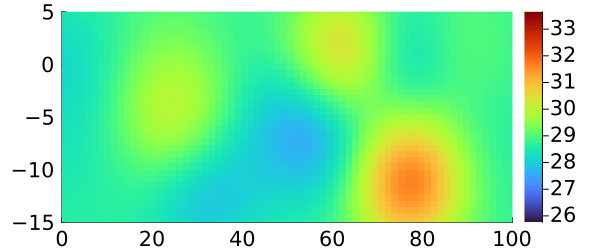}
        \caption{GT: at 40 steps}
    \end{subfigure}
    \begin{subfigure}[b]{0.485\linewidth}
        \includegraphics[width = \textwidth]{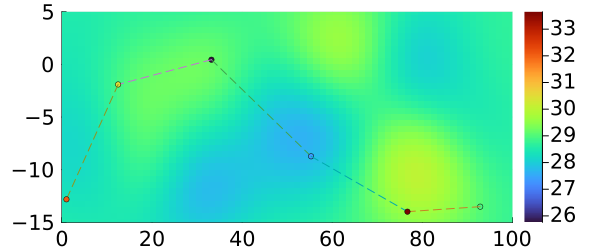}
        \caption{P: at 40 steps}
    \end{subfigure}
    \begin{subfigure}[b]{0.485\linewidth}
        \includegraphics[width = \textwidth]{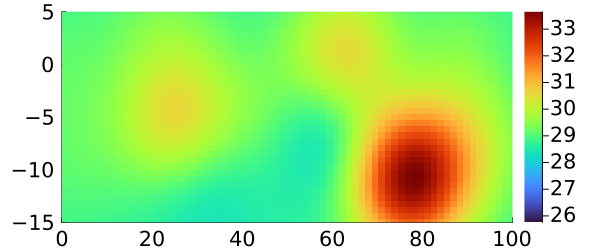}
        \caption{GT: at 60 steps}
    \end{subfigure}
    \begin{subfigure}[b]{0.485\linewidth}
        \includegraphics[width = \textwidth]{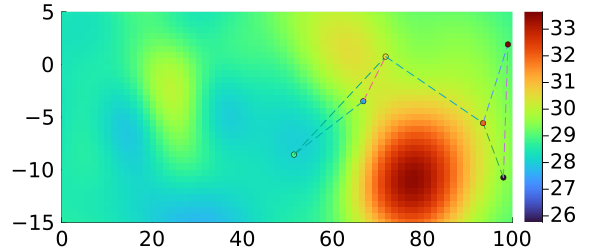}
        \caption{P: at 60 steps}
    \end{subfigure}
    \begin{subfigure}[b]{0.485\linewidth}
        \includegraphics[width = \textwidth]{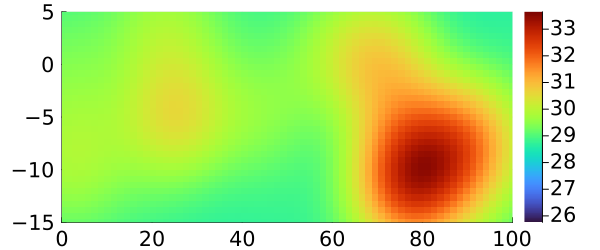}
        \caption{GT: at 80 steps}
    \end{subfigure}
    \begin{subfigure}[b]{0.485\linewidth}
        \includegraphics[width = \textwidth]{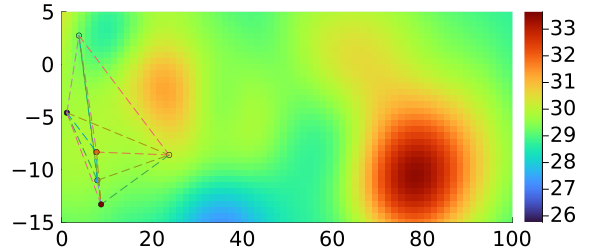}
        \caption{P: at 80 steps}
    \end{subfigure}
    \caption{Ground truth (GT) and prediction (P) of the spatio-temporal temperature.}
    \label{GT}
    \vspace{-2em}
\end{figure}

\begin{figure*}[t]
    \centering
    \includegraphics[width = 0.76\linewidth]{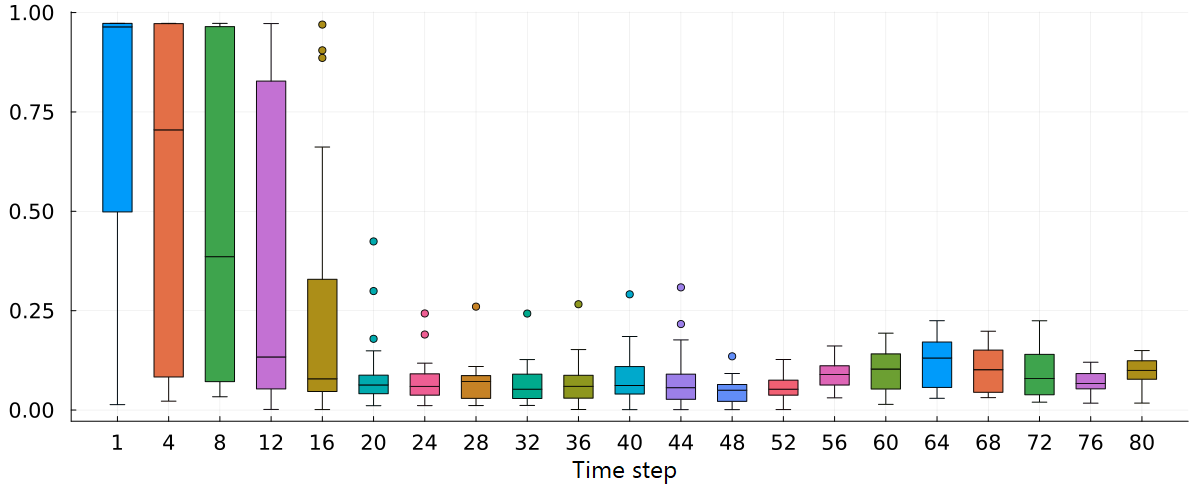}
    \caption{Box plots of prediction uncertainties at 21 test locations over time steps. Outliers (dots) exist at some time steps.}
    \label{boxp}
    \vspace{-2em}
\end{figure*}

In the IPP context, a group of robots aims to map a spatio-temporal field in an unknown environment. The robots navigate the environment while taking measurements at their moving steps. And our proposed distributed IPP algorithm provides the robots with optimal navigation in terms of gaining maximal information of the field in both space and time. In other words, the measurements taken by the robots run by our IPP approach carry most informative content of the spatio-temporal temperature field. To validate this fact, we exploited the measurements collected by the robots along their navigation paths and learned a GP model. It is noticed that this GP model was updated after every moving step of the robots as the temperature field was varying over both space and time. Since the robots could traverse to only a limited number of locations in the environment, we utilized the learned GP model to predict the temperature in the whole space at any expected time. Mappings of the spatio-temporal temperature field in the whole environment over moving steps of the robots are illustrated in the right column of Fig. \ref{GT}. For the comparison purposes, we also generated the mappings of the field by using the GT model, which are depicted in the left column of Fig. \ref{GT}. As can be seen from Fig. \ref{GT},  our method provides the comparative results that the robots could build the spatio-temporal maps intensively comparable to the ground truth.
It is also demonstrated in the right column of Fig. \ref{GT} that connectivity of the robots was well maintained overtime. Code for the simulation is written in Julia and can be found in {\tt github.com/AACLab/SpaTemIPP.git}

In practice, apart from mapping a spatio-temporal field in a whole environment, one may be interested in values of the field at some specific locations. Of course, these specific locations are not accessible by robots; hence no measurement can be made. In that case, we can use the learned GP model to temporally predict the field at those locations. To verify efficacy of our algorithm in location level, we chose 21 testing points on a grid with $X = \bbm 20, &30, &40, &50, &60, &70, &80 \ebm$ and $Y = \bbm 0, &-5 &-10 \ebm$. We exploited our learned GP model to predict the temperature at these 21 locations over 80 time steps. The prediction uncertainties at all 21 locations were summarized in a box plot. All the box plots over 80 time steps are demonstrated in Fig. \ref{boxp}. Apparently, in the first few steps when the robots did not have much information about the field, the prediction uncertainties are high. However, after about 16 time steps when the robots learned well about the field in both space and time, the uncertainties significantly reduce. Though the prediction uncertainties are considerably small from 20 time steps onwards, there are still some minor variations among them since the temperature kept changing overtime as shown in Fig. \ref{sen10}.


\section{Concluding Remarks}
This paper has addressed the problem of mapping spatio-temporal environmental field using multiple robots based on Gaussian process regression. The IPP problem has been formulated in terms of multiple prediction steps with cross-correlation cost functions that guarantee the connectivity of the robot network during the exploration time. 
By using the dual composition method, we have solved the IPP problem in a distributed manner.
The efficacy of the proposed approach was verified in a synthetic experiment utilizing a real-life dataset.
In the future works, we will consider the synchronous update of measurements in the robot network.

\bibliography{References}

\begin{thebibliography}{10}
\providecommand{\url}[1]{#1}
\csname url@samestyle\endcsname
\providecommand{\newblock}{\relax}
\providecommand{\bibinfo}[2]{#2}
\providecommand{\BIBentrySTDinterwordspacing}{\spaceskip=0pt\relax}
\providecommand{\BIBentryALTinterwordstretchfactor}{4}
\providecommand{\BIBentryALTinterwordspacing}{\spaceskip=\fontdimen2\font plus
\BIBentryALTinterwordstretchfactor\fontdimen3\font minus
  \fontdimen4\font\relax}
\providecommand{\BIBforeignlanguage}[2]{{%
\expandafter\ifx\csname l@#1\endcsname\relax
\typeout{** WARNING: IEEEtran.bst: No hyphenation pattern has been}%
\typeout{** loaded for the language `#1'. Using the pattern for}%
\typeout{** the default language instead.}%
\else
\language=\csname l@#1\endcsname
\fi
#2}}
\providecommand{\BIBdecl}{\relax}
\BIBdecl

\bibitem{dunbabin2012robots}
M.~Dunbabin and L.~Marques, ``Robots for environmental monitoring: Significant
  advancements and applications,'' \emph{IEEE Robotics \& Automation Magazine},
  vol.~19, no.~1, pp. 24--39, 2012.

\bibitem{nguyen2021mobile}
L.~Nguyen, S.~Kodagoda, R.~Ranasinghe, and G.~Dissanayake, ``Mobile robotic
  sensors for environmental monitoring using gaussian markov random field,''
  \emph{Robotica}, vol.~39, no.~5, pp. 862--884, 2021.

\bibitem{lin2020distributed}
T.~X. Lin, S.~Al-Abri, S.~Coogan, and F.~Zhang, ``A distributed scalar field
  mapping strategy for mobile robots,'' in \emph{2020 IEEE/RSJ International
  Conference on Intelligent Robots and Systems (IROS)}.\hskip 1em plus 0.5em
  minus 0.4em\relax IEEE, 2020, pp. 11\,581--11\,586.

\bibitem{yan2023confident}
R.~Yan and J.~J. Huang, ``Confident learning-based gaussian mixture model for
  leakage detection in water distribution networks,'' \emph{Water Research},
  vol. 247, p. 120773, 2023.

\bibitem{singh2010modeling}
A.~Singh, F.~Ramos, H.~D. Whyte, and W.~J. Kaiser, ``Modeling and decision
  making in spatio-temporal processes for environmental surveillance,'' in
  \emph{2010 IEEE International Conference on Robotics and Automation}.\hskip
  1em plus 0.5em minus 0.4em\relax IEEE, 2010, pp. 5490--5497.

\bibitem{ma2017informative}
K.-C. Ma, L.~Liu, and G.~S. Sukhatme, ``Informative planning and online
  learning with sparse gaussian processes,'' in \emph{2017 IEEE International
  Conference on Robotics and Automation (ICRA)}.\hskip 1em plus 0.5em minus
  0.4em\relax IEEE, 2017, pp. 4292--4298.

\bibitem{sears2022mapping}
T.~M. Sears and J.~A. Marshall, ``Mapping of spatiotemporal scalar fields by
  mobile robots using gaussian process regression,'' in \emph{2022 IEEE/RSJ
  International Conference on Intelligent Robots and Systems (IROS)}.\hskip 1em
  plus 0.5em minus 0.4em\relax IEEE, 2022, pp. 6651--6656.

\bibitem{Schmid2020traj}
L.~Schmid, M.~Pantic, R.~Khanna, L.~Ott, R.~Siegwart, and J.~Nieto, ``An
  efficient sampling-based method for online informative path planning in
  unknown environments,'' \emph{IEEE Robotics and Automation Letters}, vol.~5,
  no.~2, pp. 1500--1507, 2020.

\bibitem{nguyen2015information}
L.~V. Nguyen, S.~Kodagoda, R.~Ranasinghe, and G.~Dissanayake,
  ``Information-driven adaptive sampling strategy for mobile robotic wireless
  sensor network,'' \emph{IEEE Transactions on Control Systems Technology},
  vol.~24, no.~1, pp. 372--379, 2015.

\bibitem{le2021efficient}
V.-A. Le, L.~Nguyen, and T.~X. Nghiem, ``An efficient adaptive sampling
  approach for mobile robotic sensor networks using proximal admm,'' in
  \emph{2021 American Control Conference (ACC)}.\hskip 1em plus 0.5em minus
  0.4em\relax IEEE, 2021, pp. 1101--1106.

\bibitem{ding2024resource}
T.~Ding, R.~Zheng, S.~Zhang, and M.~Liu, ``Resource-efficient cooperative
  online scalar field mapping via distributed sparse gaussian process
  regression,'' \emph{IEEE Robotics and Automation Letters}, 2024.

\bibitem{binh2024ipp}
B.~Nguyen, T.~X. Nghiem, L.~Nguyen, H.~M. La, and T.~Nguyen,
  ``Connectivity-preserving distributed informative path planning for mobile
  robot networks,'' \emph{IEEE Robotics and Automation Letters}, vol.~9, no.~3,
  pp. 2949--2956, 2024.

\bibitem{rizzo1994characterization}
D.~M. Rizzo and D.~E. Dougherty, ``Characterization of aquifer properties using
  artificial neural networks: Neural kriging,'' \emph{Water Resources
  Research}, vol.~30, no.~2, pp. 483--497, 1994.

\bibitem{williams2006gaussian}
C.~K. Williams and C.~E. Rasmussen, \emph{Gaussian processes for machine
  learning}.\hskip 1em plus 0.5em minus 0.4em\relax MIT press Cambridge, MA,
  2006, vol.~2, no.~3.

\bibitem{cover1999elements}
T.~M. Cover, \emph{Elements of information theory}.\hskip 1em plus 0.5em minus
  0.4em\relax John Wiley \& Sons, 1999.

\bibitem{nguyen2023real}
B.~Nguyen, T.~Nghiem, L.~Nguyen, A.~T. Nguyen, and T.~Nguyen, ``Real-time
  distributed trajectory planning for mobile robots,''
  \emph{IFAC-PapersOnLine}, vol.~56, no.~2, pp. 2152--2157, 2023.

\bibitem{rush2012tutorial}
A.~M. Rush and M.~Collins, ``A tutorial on dual decomposition and lagrangian
  relaxation for inference in natural language processing,'' \emph{Journal of
  Artificial Intelligence Research}, vol.~45, pp. 305--362, 2012.

\bibitem{stellato2020osqp}
B.~Stellato, G.~Banjac, P.~Goulart, A.~Bemporad, and S.~Boyd, ``Osqp: An
  operator splitting solver for quadratic programs,'' \emph{Mathematical
  Programming Computation}, vol.~12, no.~4, pp. 637--672, 2020.

\bibitem{domahidi2013ecos}
A.~Domahidi, E.~Chu, and S.~Boyd, ``Ecos: An socp solver for embedded
  systems,'' in \emph{2013 European control conference (ECC)}.\hskip 1em plus
  0.5em minus 0.4em\relax IEEE, 2013, pp. 3071--3076.

\bibitem{dataset_soil}
Universidad-Loyola-ODS-Research-Group, \emph{Soil moisture and temperature data
  in agricultural soil}.\hskip 1em plus 0.5em minus 0.4em\relax Mendeley Data,
  V1, doi: 10.17632/fpbfmc9vnm.1, 2022.

\end{thebibliography}
\bibliographystyle{IEEEtran}

\end{document}